\documentclass{article}
\usepackage{ijcai23}

\usepackage{times}
\usepackage{soul}
\usepackage{url}
\usepackage[hidelinks]{hyperref}
\usepackage[utf8]{inputenc}
\usepackage[small]{caption}
\usepackage{graphicx}
\usepackage{amsmath}
\usepackage{amssymb}
\usepackage{amsthm}
\usepackage{booktabs}
\usepackage{algorithmic}
\usepackage[switch]{lineno}
\usepackage{comment}
\usepackage{subcaption}
\usepackage[ruled,vlined,linesnumbered]{algorithm2e}

\newtheorem{definition}{Definition}

\newtheorem{theorem}{Theorem}
\newtheorem{proposition}{Proposition}
\newtheorem{corollary}{Corollary}

\title{Counterexample Guided Abstraction Refinement with Non-Refined Abstractions for Multi-Agent Path Finding}

\author{
   Pavel Surynek
   \affiliations
   Faculty of Information Technology,
   Czech Technical University in Prague\\
   Th\'{a}kurova 9, 160 00 Praha 6, 
   Czechia
   \emails
   pavel.surynek@fit.cvut.cz
}

\begin{document}

\maketitle

\begin{abstract}
Counterexample guided abstraction refinement (CEGAR) represents a powerful symbolic technique for various tasks such as model checking and reachability analysis. Recently, CEGAR combined with Boolean satisfiability (SAT) has been applied for multi-agent path finding (MAPF), a problem where the task is to navigate agents from their start positions to given individual goal positions so that the agents do not collide with each other.

The recent CEGAR approach used the initial abstraction of the MAPF problem where collisions between agents were omitted and were eliminated in subsequent abstraction refinements. We propose in this work a novel CEGAR-style solver for MAPF based on SAT in which some abstractions are deliberately left non-refined. This adds the necessity to post-process the answers obtained from the underlying SAT solver as these answers slightly differ from the correct MAPF solutions. Non-refining however yields order-of-magnitude smaller SAT encodings than those of the previous approach and speeds up the overall solving process making the SAT-based solver for MAPF competitive again in relevant benchmarks.
\end{abstract}


\section{Introduction}
Multi-agent path finding (MAPF) \cite{DBLP:conf/aiide/Silver05,DBLP:journals/jair/Ryan08,DBLP:conf/icra/Surynek09,DBLP:journals/jair/WangB11,DBLP:journals/ai/SharonSGF13} is a task of finding non-conflicting paths for $k \in \mathbb{N}$ agents $A=\{a_1,a_2,...,a_k\}$ that move in an undirected graph $G=(V,E)$ across its edges such that each agent reaches its goal vertex from the given start vertex via its path. Starting configuration of agents is defined by a simple assignment $s: A \rightarrow V$ and the goal configuration is defined by a simple assignment $g: A \rightarrow V$. A conflict between agents is usually defined as simultaneous occupancy of the same vertex by two or more agents or as a traversal of an edge by agents in opposite directions. Although MAPF started as purely theoretical studies of graph pebbling and puzzle solving \cite{DBLP:conf/focs/KornhauserMS84,DBLP:conf/aaai/RatnerW86,DBLP:journals/jsc/RatnerW90}, it has grown into artificial intelligence mainstream topic with a significant impact on many fields including warehouse logistics \cite{DBLP:conf/atal/LiTKDKK20}.

Many problems in robotics \cite{DBLP:conf/robovis/ChudyPS20,DBLP:journals/ras/WenLL22,DBLP:conf/iros/GharbiCS09}, urban traffic optimization \cite{DBLP:conf/socs/AtzmonDR19,DBLP:conf/atal/HoSGGCP19}, FPGA circuit design \cite{DBLP:conf/sbac-pad/NerySG17}, and computer games \cite{DBLP:conf/cig/SigurdsonB0HK18} can be regarded from the perspective of MAPF as listed by various surveys including \cite{DBLP:conf/socs/FelnerSSBGSSWS17,DBLP:journals/corr/0001KA0HKUXTS17}.

Particular aspect that is motivated by practice and makes MAPF challenging is the need to find the optimal solutions with respect to some cumulative cost \cite{DBLP:conf/aaai/YuL13}. Commonly used cumulative costs in MAPF are {\em makespan} and {\em sum-of-costs}\cite{DBLP:conf/raai/Stern19}. The makespan corresponds to the length of the longest agent's path. The sum-of-costs is the sum of costs of individual paths which corresponds to the sum of unit costs of actions, the wait actions including. An example of MAPF problem and its sum-of-costs optimal solution is shown in Figure \ref{figure-MAPF}.


\begin{figure}[h]
    \centering
    \includegraphics[trim={4.0cm 22.8cm 3.5cm 3.0cm},clip,width=0.5\textwidth]{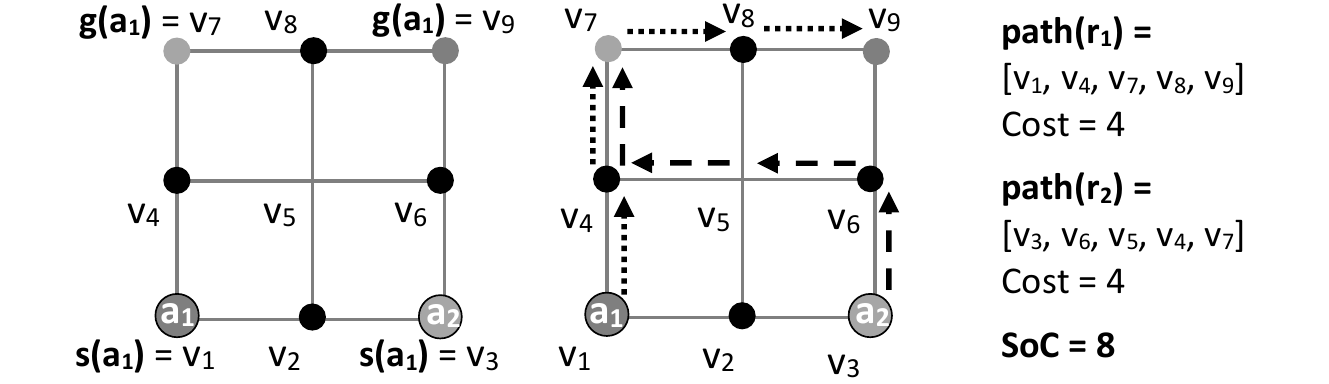}
    \caption{Multi-agent path finding (MAPF) with agents $a_1$ and $a_2$.}
    \label{figure-MAPF}
\end{figure}

We address MAPF from the perspective of compilation techniques that represent a major alternative to search-based solvers \cite{DBLP:conf/aiide/Silver05,Wagner-2011-7385,DBLP:journals/ai/SharonSGF13,DBLP:journals/ai/SharonSFS15} for MAPF. Compilation-based solvers reduce the input MAPF instance to an instance in a different well established formalism for which an efficient solver exists. Such formalisms are for example {\em constraint programming} (CSP) \cite{DBLP:books/daglib/0016622,DBLP:conf/icra/Ryan10}, {\em Boolean satisfiability} (SAT) \cite{biere2021handbook,DBLP:conf/ictai/Surynek12}, or {\em mixed integer linear programming} (MILP) \cite{rader2010deterministic,DBLP:conf/ijcai/LamBHS19}.

The basic compilation scheme for sum-of-costs optimal MAPF solving has been introduced by the MDD-SAT \cite{DBLP:conf/ecai/SurynekFSB16} solver that uses so called {\em complete models} to compile MAPF instances into SAT (see Figure \ref{figure-COMPILATION}). The target Boolean formula of the complete model is satisfiable if and only if the input MAPF has a solution of a specified sum-of-costs. The complete model as introduced in MDD-SAT consists of three group of constraints:

\begin{itemize}
\item {\bf Agent propagation} constraints - these constraints ensure that if an agent appears in vertex $v$ at time step $t$ then the agent appears in some neighbor of $v$ (including $v$) at time step $t+1$. The side effect of these constraints is that the agent never disappears. Cost calculation and bounding is done together with agent propagation.
\item {\bf Path consistency} constraints - these constrains ensure that agents move along proper paths, that is, agents do not multiply and do not appear spontaneously.
\item {\bf Conflict elimination} constraints - to ensure that agents do not conflict with each other according to the MAPF rules (vertex and edge conflict).
\end{itemize}

\begin{figure}[h]
    \centering
    \includegraphics[trim={2.1cm 22.5cm 3.5cm 3.0cm},clip,width=0.5\textwidth]{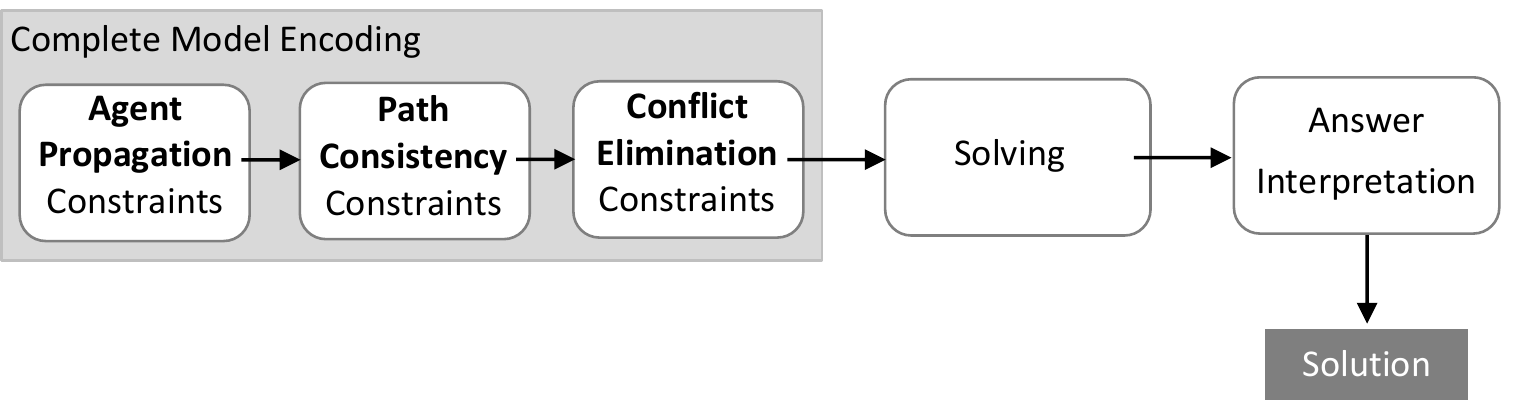}
    \caption{Schematic diagram of the basic MAPF compilation with complete model.}
    \label{figure-COMPILATION}
\end{figure}

A significant improvement over complete models in problem compilation for MAPF is the introduction of laziness via {\em incomplete models} in the CSP-based LazyCBS \cite{DBLP:conf/aips/GangeHS19}, SAT-based SMT-CBS \cite{DBLP:conf/ijcai/Surynek19}, and MILP-based BCP \cite{DBLP:conf/ijcai/LamBHS19}. These solvers use incomplete models of MAPF where the conflict elimination constraints are omitted for which equivalent solvability no longer holds, but only the implication: if the MAPF instance is solvable then the instance in the target formalism is solvable too.

The discrepancy between the original formulation of MAPF and its compiled variant in the target formalism is eliminated by abstraction refinements similarly as it is done in the {\em counterexample guided abstraction refinement} (CEGAR) \cite{DBLP:conf/cav/ClarkeGJLV00} approach for model checking (see Figure \ref{figure-CEGAR}). Specifically in MAPF the counterexamples are represented by conflicts between agents and their refinements correspond to elimination of these conflicts.

\subsection{Contribution}

We further generalize the CEGAR approach for MAPF by further omitting the {\em path consistency} constraints. Hence only agent propagation constraints remain in the {\em initial abstraction}. In addition to this, in the abstraction refinement we do not try to refine with respect to all the omitted constraints (path consistency and conflict elimination). Instead we only generate counterexamples for conflicts and refine the conflicts while path consistency remains non-refined - the abstraction refinement in our case is inherently incomplete. To mitigate the impact of non-refined path consistency constraints we add a new step in the CEGAR compilation architecture in which we post-process the answer from the SAT solver.

The omitted path consistency constraints lead to finding directed acyclic sub-graphs (DAGs) that connects agents' initial and goal positions instead of proper paths. The desired path for agents need to be extracted from the sub-graphs in the new post-processing step.

Our contribution can be understood in a broader perspective as a generalization how problems are compiled. In the classic approaches such as SATPlan \cite{DBLP:conf/ecai/KautzS92} the problem solution is read from the assignment of decision variables directly. In our approach the assignment of decision variables does not represent the problem solution directly, but the solution needs to be extracted from the assignment by non-trivial process, though this process should be fast (polynomial time) to keep the problem compilation viable.

\section{Background}

One of the first uses of problem compilation can be seen in the SATPlan algorithm \cite{DBLP:conf/ecai/KautzS92,DBLP:conf/ijcai/KautzS99} for classical planning \cite{DBLP:books/daglib/0014222}. A Boolean formula that is satisfiable if and only if a plan of specified length exists is constructed and checked for satisfiability by the SAT solver. To guarantee finding a plan of the minimum length, the SATPlan planner iteratively consults formulae encoding the existence of a plan of length $l_0$, $l_0+1$, $l_0+2$,... where $l_0$ is a lower bound for the plan length, until the first satisfiable formula is found.

Similar approach has been used in SAT-based solvers for MAPF such as MDD-SAT \cite{DBLP:conf/ecai/SurynekFSB16} where the SAT solver is consulted about the existence of a solution for the given MAPF instance of specified sum-of-cost $\mathit{SoC}$. To ensure finding sum-of-costs optimal solution for solvable MAPF, MDD-SAT checks using the SAT solver existence of solutions for $\mathit{SoC}_0$, $\mathit{SoC}_0 + 1$, $\mathit{SoC}_0 + 2$, ... until the first positive answer which due to monotonicity of existence of solutions w.r.t. bounded sum-of-costs is guaranteed to correspond to the optimal sum-of-cost.

\subsection{The MAPF Encoding as SAT}

The complete model for the sum-of-costs optimal MAPF as SAT built-in the MDD-SAT solver uses Boolean decision variables inspired by {\em direct encoding} \cite{DBLP:conf/cp/Walsh00} and {\em multi-valued decision diagrams} \cite{DBLP:conf/cp/AndersenHHT07}. We briefly summarize the complete model as introduced by MDD-SAT.

The decision variables need to represent all possible paths of agents such that their sum-of-costs is at most given value $\mathit{SoC} > 0$. Since it is possible for an agent to visit a single vertex multiple times, a so-called {\em time expansion} \cite{DBLP:journals/amai/Surynek17} of the underlying graph $G$ is constructed for each agent denoted $\mathit{TEG}_i$.

$\mathit{TEG}_i$ is defined for a given maximum length of individual agent's path $T \in \mathbb{N}$ consisting of $T+1$ copies of vertices of $G$ called {\em layers} indexed by $0,1,...,T$, where $\{(v_1,t),(v_2,t),...,(v_n,t)\}$ are nodes of the $t$-th layer of $\mathit{TEG}_i$. The layers correspond to individual time-steps and are interconnected by directed edges that model possible transitions of agents, that is there is a directed edge $[(v_j,t),(v_{j'},t+1)]$ whenever there is an edge $\{v_j,v_{j'}\} \in E$. Directed edges $[(v_j,t),(v_{j'},t+1)]$ are added to model wait actions of agents.

The following proposition establishes the correspondence between the actual paths traveled by agents in $G$ \footnote{The proper terminology for agents' paths should be trajectories or just sequences of vertices as a vertex can visited multiple times by the agent which does not correspond to the standard graph terminology where a path is a simple sequence of vertices.} and directed paths in TEGs.

\begin{proposition}
Any path of length at most $T$ of agent $a_i$ going from $s(a_i)$ to $g(a_i)$ is represented by a directed path in $\mathit{TEG}_i$ connecting $(s(a_i),0)$ and $(g(a_i),T)$.
\label{prop-TEG}
\end{proposition}

Having the proposition, we can speak about a representation of agent's path (trajectory) in TEG.

Since not all nodes $(v_j,t)$ of $\mathit{TEG}_i$ are actually reachable considering the maximum agent's path length $T$ and its start $s(a_i)$ and goal $g(a_i)$ because either $v_j$ is farther than $t$ steps from $s(a_i)$ or farther than $T-t$ steps from $g(a_i)$, such nodes be can pruned from $\mathit{TEG}_i$ without compromising the representation of all relevant paths. $\mathit{TED}_i$ after pruning unreachable nodes is called a {\em multi-valued decision diagram} for agent $a_i$ and is denoted $\mathit{MDD}_i$ (nodes and edges of $\mathit{MDD}_i$ are denoted $\mathit{MDD}_i.V$ and $\mathit{MDD}_i.E$ respectively). Example MDD is shown in Figure \ref{figure-MDD}. The complete model for MAPF in the MDD-SAT solver is derived from MDDs.

The decision variable denoted $\mathcal{X}_{i,v_j}^t$ is introduced for each node $(v_j,t) \in \mathit{MDD}_i$ and expresses that agent $a_i$ appears in vertex $v_j$ at time step $t$ \footnote{Auxiliary variables $\mathcal{E}_{i,v_j,v_{j'}}^t$ to express that agent $a_i$ moves from $v_j$ to $v_{j'}$ between time-steps $t$ and $t+1$ can be used as well, but we will base our description only on the $\mathcal{X}_{i,v_j}^t$ variables as the $\mathcal{E}_{i,v_j,v_{j'}}^t$ variables are direct consequence of of the $\mathcal{X}_{i,v_j}^t$ variables: $\mathcal{E}_{i,v_j,v_{j'}}^t  \leftrightarrow \mathcal{X}_{i,v_j}^t \wedge \mathcal{X}_{i,v_{j'}}^{t+1}$.}.

\begin{figure}[h]
    \centering
    \includegraphics[trim={3.0cm 22.2cm 4.5cm 3.5cm},clip,width=0.5\textwidth]{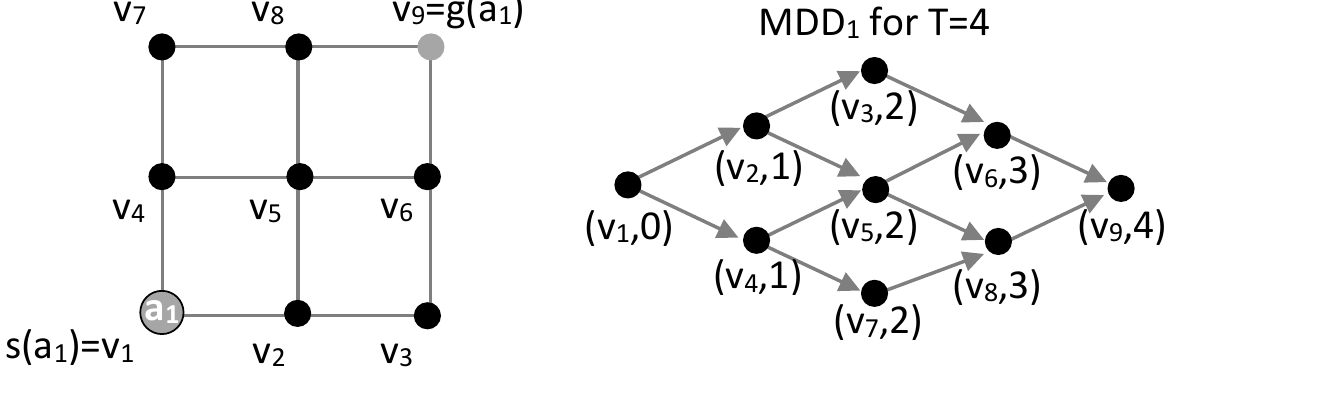}
    \caption{Multi-valued decision diagram for the path length $T=4$.}
    \label{figure-MDD}
\end{figure}

The three groups of constraints expressed on top of the $\mathcal{X}_{i,v_j}^t$ variables are as follows:

\subsubsection{Agent Propagation Constraints}

\begin{equation}
 {  \mathcal{X}_{i,v_j}^t \rightarrow \bigvee_{j'\;|\;[(v_j,t);(v_{j'},t+1)] \in \mathit{MDD}_i.E}{\mathcal{X}_{i,v_{j'}}^{t+1}}
 }
 \label{eq-prop-1}
\end{equation}

This constraint is introduced for each $a_i$, $v_j$, and $t$ and specifies that agent must proceed to the next level in MDD. In addition, to this we account among these constraints also calculation and bounding of the cost. For each time-step $t$ in MDD an auxiliary variable is introduced which is $\mathit{TRUE}$ if and only if the agent performed an action. The cost bounding over auxiliary variables is carried out through {\em cardinality constraints} encoded as Boolean circuits into the formula \cite{DBLP:conf/cp/BailleuxB03,DBLP:conf/cp/SilvaL07}.

\begin{equation}
 {  \mathcal{X}_{i,s(a_i)}^0 \wedge \bigwedge_{j\;|\;v_j \neq s(a_i)}{\neg \mathcal{X}_{i,v_j}^0}
 }
 \label{eq-prop-2}
\end{equation}

\begin{equation}
 {  \mathcal{X}_{i,g(a_i)}^T \wedge \bigwedge_{j\;|\;v_j \neq g(a_i)}{\neg \mathcal{X}_{i,v_j}^T}
 }
 \label{eq-prop-3}
\end{equation}

These equations are introduced for each agent $a_i$ and ensure that agent $a_i$ starts in vertex $s(a_i)$ at time-step 0 and finished in its goal vertex $g(a_i)$ at time-step $T$.

\subsubsection{Path Consistency Constraints}

\begin{equation}
 {  \sum_{j \;|\; (v_j,t) \in \mathit{MDD}_i.V}{\mathcal{X}_{i,v_j}^t} = 1
 }
 \label{eq-cons-1}
\end{equation}

This constraint is introduced for each $a_i$ and $t$ and specifies that agent can appear in exactly one vertex at a time.

\subsubsection{Conflict Elimination Constraints}

\begin{equation}
 {  \sum_{i \;|\; (v_j,t) \in \mathit{MDD}_i.V}{\mathcal{X}_{i,v_j}^t} \leq 1
 }
 \label{eq-conf-1}
\end{equation}

This constraint is introduced for each $v_j$ and $t$ and specifies that at most one agent can reside in vertex $v_j$ at time $t$. The constraint eliminates vertex conflicts, edge conflicts can be eliminated analogously.

As some of the constraint are defined by pseudo-Boolean expression, proper translation to CNF is needed which most notably concerns the {\em at-most-one} constraints \cite{Chen2010ANS,DBLP:conf/soict/NguyenM15}. The combination of {\em pair-wise} encoding and {\em sequential counter} are used in the MDD-SAT and SMT-CBS solvers.

The formula collecting the above constraints is being built for the specific {\em sum-of-costs} $\mathit{SoC}$ and is denoted $\mathcal{F}(\mathit{SoC})$.
The completeness of the model encoded by $\mathcal{F}(\mathit{SoC})$ can be summarized as follows \cite{DBLP:conf/ecai/SurynekFSB16}:

\begin{proposition}
The input MAPF has a solution of the given sum-of-costs $\mathit{SoC} \Leftrightarrow \mathcal{F}(\mathit{SoC})$ is satisfiable.
\end{proposition}

The core of the proof of the proposition is a correspondence of satisfying assignments of $\mathcal{F}(\mathit{SoC})$ and directed paths in MDDs which in turn due to Proposition \ref{prop-TEG} establishes a correspondence between agents' paths and satisfying assignments of $\mathcal{F}(\mathit{SoC})$.

\subsection{CEGAR for MAPF}

The next step in compilation-based MAPF solving is represented by the integration of ideas from {\em counterexample guided abstraction and refinement} (CEGAR)  \cite{DBLP:conf/cav/ClarkeGJLV00,DBLP:conf/cade/Clarke03}. The two representative solvers SMT-CBS \cite{DBLP:conf/ijcai/Surynek19} based on SAT and Lazy-CBS \cite{DBLP:conf/aips/GangeHS19} based on CSP resolve conflicts between agents as done in the CEGAR approach (though the authors of these MAPF solvers do not explicitly mention CEGAR).

\begin{figure}[h]
    \centering
    \includegraphics[trim={2.7cm 20.2cm 2.8cm 2.9cm},clip,width=0.5\textwidth]{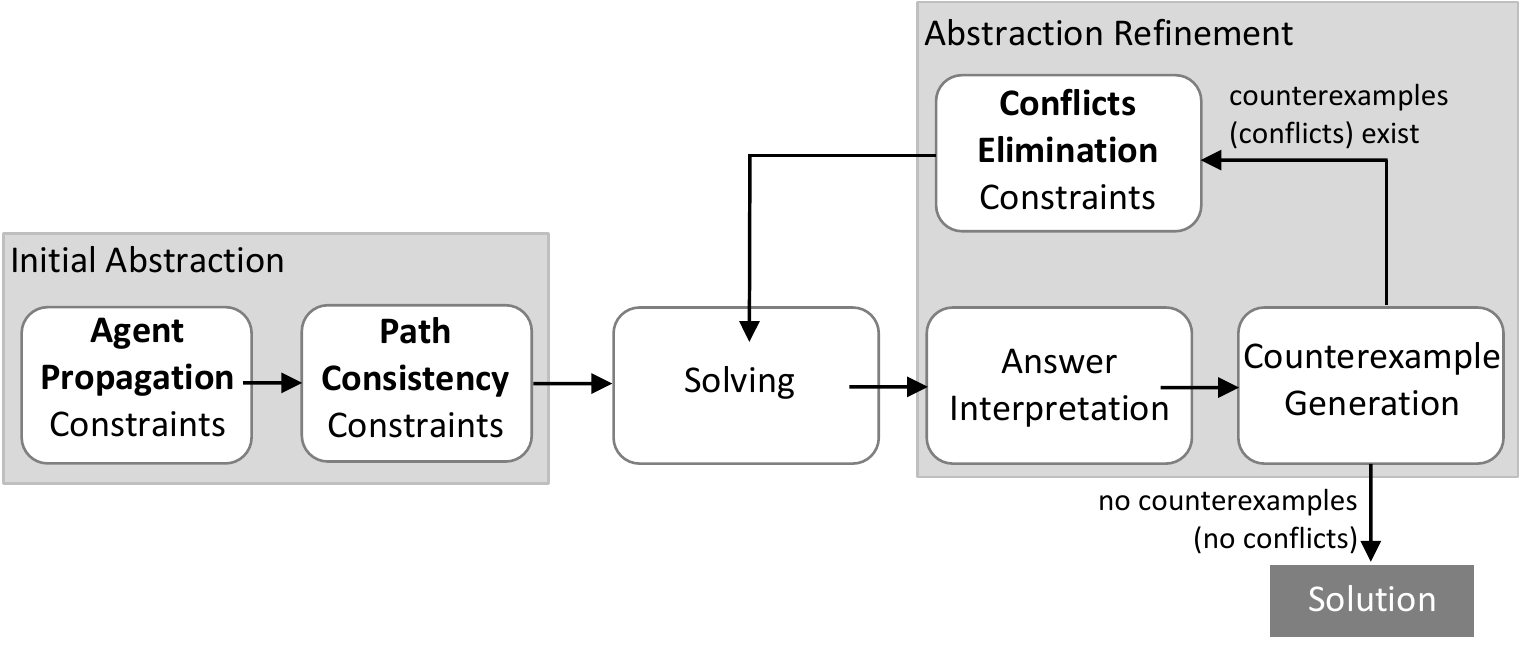}
    \caption{Schematic diagram of counterexample guided abstraction refinement (CEGAR) for MAPF. Conflicts between agents are treated as counterexamples and eliminated in the abstraction refinement loop. }
    \label{figure-CEGAR}
\end{figure}

The general CEGAR approach for compilation-based problem solving starts with a so called {\em initial abstraction} of the problem instance being solved in some target formalism such as SAT \cite{biere2021handbook} or CSP \cite{DBLP:books/daglib/0016622}. The initial abstraction do not model the input instance in the full details. However still the initial abstraction is passed to the solver for the target formalism despite the solver is not provided all the details needed to solve it. Then the solver will come with some answer and since it could not take some details into account during the solving phase, the answer must be checked, usually against full details of how the problem instance is defined. Two cases need to be distinguished at this stage. If the provided answer matches the instance definition then it is returned and the solving process finishes. Otherwise the CEGAR solving process generates {\em counterexample} that is determined by the mismatch between the provided answer and the requirements the expected answer should satisfy. This mismatch is usually represented by the violation of some constraints that were not expressed in the abstraction. Then the solving process continues with a so called {\em abstraction refinement} in which the abstraction is augmented to eliminate the counterexample and the solving process continues with the next iteration of (now refined) abstraction solving.

To keep the problem solving in the CEGAR approach sound, the abstractions must fulfill certain basic requirements such as if the problem instance is solvable then any of its abstractions should be solvable as well (the opposite does not hold: the solvable abstraction does not necessarily imply that the input instance is solvable).

The CEGAR approach for MAPF as implemented in SMT-CBS and Lazy-CBS is illustrated in Figure \ref{figure-CEGAR}. The initial abstraction in both algorithms models existence of paths for individual agents but omits the requirement to avoid conflicts between agents. The abstraction refinement hence must check the answers of the target solver against the definition of conflicts in MAPF. If a conflict is found then the abstraction is refined so that the conflict is eliminated.

The abstraction refinement for a conflict, say between agents $a_i$ and $a_j$ in vertex $v$ at time-step $t$, is done in the case of CSP-based Lazy-CBS by adding a fresh finite domain variable $p_{v,t} \in A$ that indicates what agent occupies vertex $v$ at time step $t$. The need to assign the variable $p_{v,t}$ a single value eliminates the conflict.

The SAT-based SMT-CBS algorithm refines the conflict by adding a new constraint over the existing Boolean variables that forbids the occupation of vertex $v$ at time-step $t$ by agents $a_i$ and $a_j$ simultaneously:  $\neg \mathcal{X}_{i,v}^{t} \vee \neg \mathcal{X}_{j,v}^{t}$.

The edge conflicts and potentially different kinds of conflicts in MAPF and its variants \cite{DBLP:conf/ijcai/AndreychukYAS19,DBLP:journals/algorithmica/BonnetMR18} can be treated analogously.

The abstraction at any point in SMT-CBS and Lazy-CBS forms a so called {\em incomplete model} for MAPF. Unlike the complete model from MDD-SAT, the equivalent-solvability of the model in the target formalism and the input MAPF instance does not hold for the incomplete MAPF model. Let us express this property formally for the SAT-based formulation. Let $\mathcal{F'}(\mathit{SoC})$ be the formula obtained at some stage in CEGAR loop of SMT-CBS used for answering the question whether there is a solution of the input MAPF instance of the given sum-of-costs $\mathit{SoC}$. Then the following proposition holds \cite{DBLP:conf/ijcai/Surynek19}:

\begin{proposition}
The input MAPF has a solution of the given sum-of-costs $\mathit{SoC}$ $\Rightarrow$ $\mathcal{F'}(\mathit{SoC})$ is satisfiable.
\end{proposition}

The incompleteness property also represents the requirement that keeps the CEGAR approach sound for MAPF as it says that $\mathcal{F'}(\mathit{SoC})$ is an abstraction for MAPF.

\section{Non-Refined Abstractions in MAPF}

We are going further in the CEGAR architecture of the MAPF solver. In addition to {\em conflict elimination} constrains we also omit {\em path consistency} constraints in the initial abstraction. Moreover we never make any refinement with respect to the omitted {\em path consistency} constraints - the corresponding abstraction remains {\bf non-refined}.

We call the new algorithm based on the non-refined abstractions {\em Non-Refined SAT} or {\bf NRF-SAT} in short. The pseudo-code of NRF-SAT is shown as Algorithm \ref{alg-NRF-SAT}.

The high-level function of the algorithm NRF-SAT-MAPF() is analogous to SATPlan or MDD-SAT main loop where search for the optimal sum-of-costs is done by trying to answer questions whether there is a solution of MAPF of a specified sum-of-costs $\mathit{SoC}$ (lines 5-6). To answer these questions, the algorithm uses low level function NRF-SAT-Bounded() that implements the CEGAR architecture with non-refined abstractions as shown in Figure \ref{figure-NRF-CEGAR}.

As the satisfying assignment of formula $\mathcal{F''}$ representing the incomplete model in NRF-SAT does not correspond to paths for agents (line 15) further post-processing of the SAT solver's answer is needed (line 16). As we will see later, the answer obtained directly from the SAT solver can be interpreted as special directed acyclic graph (DAG) that connects agent's start vertex with its goal vertex. Hence the post-processing consists in extraction of proper agent's path from the DAG.

The rest of the low-level loop (lines 17-22) represents abstraction refinements with respect to conflicts between agents. Let us note that conflicts being discovered are collected and reused in the next iteration of the high-level loop.

\begin{figure}[h]
    \centering
    \includegraphics[trim={2.4cm 20.2cm 2.4cm 3.0cm},clip,width=0.5\textwidth]{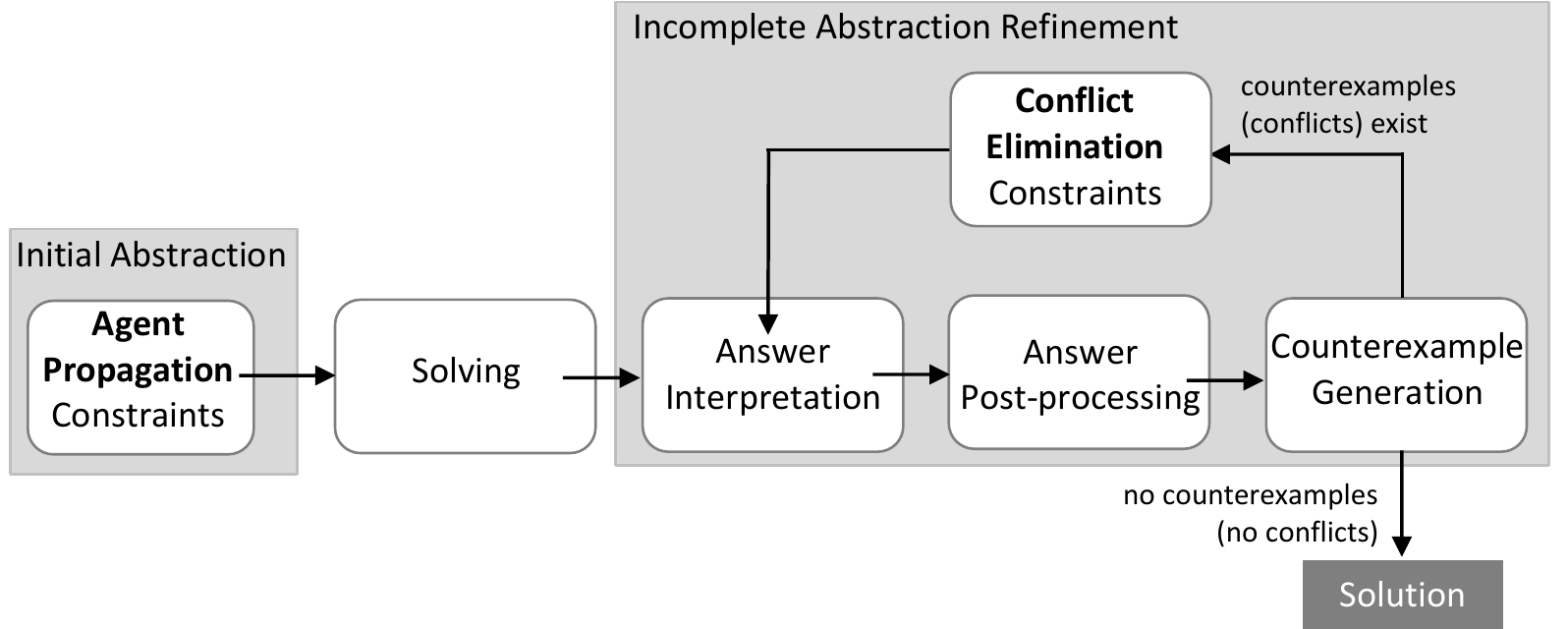}
    \caption{Schematic diagram of CEGAR problem solver for MAPF with non-refined abstractions.}
    \label{figure-NRF-CEGAR}
\end{figure}

The NRF-SAT algorithm is sound and optimal, more precisely it returns a sum-of-costs optimal solution for a solvable input MAPF instance. The following series of propositions shows the claim.

\begin{definition}
$\mathit{DAG}_i$ is a directed acyclic sub-graph of $\mathit{MDD}_i$ such that $(g(a_i),T) \in \mathit{DAG}_i.V$ and there exists a directed path from any $(v_j,t) \in \mathit{DAG}_i.V$ to $(g(a_i),T)$.
\end{definition}

\begin{proposition}
The interpretation of satisfying assignment of $\mathcal{F''}(\mathit{SoC})$ at any stage of the NRF-SAT algorithm corresponds to $\mathit{DAG}_i$ such that $(s(a_i), 0) \in \mathit{DAG}_i.V$.
\label{prop-dag}
\end{proposition}

\begin{proof}
Assume that Boolean decision variables of $\mathcal{X}_{i,v_j}^t$ are set to $\mathit{TRUE}$ to reflect the choice of vertices in $\mathit{MDD}_i$ by a given $\mathit{DAG}_i$. Then the agent propagation constraints (\ref{eq-prop-1}) and (\ref{eq-prop-3}) are satisfied since they directly correspond to existence of paths towards $(g(a_i),T)$ from any node selected by $\mathit{DAG}_i$ which is satisfied by the definition of  $\mathit{DAG}_i$. Conversely, any assignment of Boolean decision variables that satisfies constraints  (\ref{eq-prop-1}) and (\ref{eq-prop-3}) corresponds to $\mathit{DAG}_i$ since any setting of $\mathcal{X}_{i,v_j}^t$ to $\mathit{TRUE}$ must be propagated via (\ref{eq-prop-1}) and (\ref{eq-prop-3}) towards $\mathcal{X}_{i,g(a_i)}^T$. In addition to this, the constraint (\ref{eq-prop-2}) ensures that $(s(a_i),0) \in \mathit{DAG}_i.V$.
\end{proof}

The immediate corollary of Proposition \ref{prop-dag} and the definition of $\mathit{DAG}_i$ is as follows:

\begin{corollary}
There exists a directed path connecting $(s(a_i),0)$ and $(g(a_i),T)$ in $\mathit{DAG}_i$ for each agent $a_i$ obtained from satisfying assignment of $\mathcal{F''}(\mathit{SoC})$.
\end{corollary}

\begin{algorithm}[t]
\begin{footnotesize}
\SetKwBlock{NRICL}{NRF-SAT-MAPF($\mathcal{M}=(G,A,s,g)$)}{end} \NRICL{
    $\mathit{SoC} \gets$ lower-Bound($\mathcal{M}$)\\
    $\mathit{conflicts} \gets \emptyset$\\
    \While {$\mathit{TRUE}$} {   
        $(\mathit{paths,conflicts}) \gets$ \\ NRF-SAT-Bounded($\mathcal{M}$, $\mathit{SoC}$,$\mathit{conflicts}$)\\
        \If {$\mathit{paths} \neq \mathit{UNSAT}$}{
            \Return $\mathit{paths}$
        }
        $\mathit{SoC} \gets \mathit{SoC} + 1$\\
     }
}

\SetKwBlock{NRICL}{NRF-SAT-Bounded($\mathcal{M}$,$\mathit{SoC}$,$\mathit{conflicts}$)}{end} \NRICL{
      $\mathcal{F''} \gets$ build-Initial-Abstraction($\mathcal{M}$,$\mathit{SoC}$,$\mathit{conflicts}$)\\
	\While {$\mathit{TRUE}$} {        
          $\mathit{assignment} \gets$ consult-SAT-Solver($\mathcal{F''}$)\\
          
          \If {$\mathit{assignment} \neq \mathit{UNSAT}$}{
            $\mathit{DAGs} \gets$ interpret($\mathcal{M}$,$\mathit{assignment}$)\\
            $\mathit{paths} \gets$ extract-Paths($\mathcal{M}$,$\mathit{DAGs}$)\\           
            $\mathit{conflicts'} \gets$ validate($\mathcal{M}$,$\mathit{paths})$\\
            \If{$\mathit{conflicts'} = \emptyset$}{
                \Return $(\mathit{paths}, \mathit{conflicts})$
            }
                       
            \For{each $c \in \mathit{conflicts'}$}{
              $\mathcal{F'' }\gets \mathcal{F''} \cup$ eliminate-Conflict($c$)\\
            }
            $\mathit{conflicts} \gets \mathit{conflicts} \cup \mathit{conflicts'}$
          }
          \Return {($\mathit{UNSAT}$,$\mathit{conflicts}$)}\\
      }
}

\caption{NRF-SAT: MAPF solving via CEGAR with non-refined abstractions.} \label{alg-NRF-SAT}
\end{footnotesize}
\end{algorithm}

This directed path will be extracted from the SAT solver answer during the answer post-processing step as illustrated in Figure \ref{figure-path-extraction}.

Additional constraints that bound the cost and those that eliminate conflicts included during abstraction refinements restrict the set of DAGs that can correspond to satisfying assignments of $\mathcal{F''}(\mathit{SoC})$.

The important property of $\mathit{DAG}_i$ that directly follows from its definition is that a path for agent $a_i$ of length $T$ going from its start vertex $s(a_i)$ to its goal $g(a_i)$ can be represented as a $\mathit{DAG}_i$ (we only add time indices to vertices visited by the agent along the path to obtain the $\mathit{DAG}_i$). Hence $\mathcal{F''}(\mathit{SoC})$ is an incomplete model (an abstraction) for MAPF:

\begin{proposition}
The input MAPF has a solution of the given sum-of-costs $\mathit{SoC} \Rightarrow \mathcal{F''}(\mathit{SoC})$ is satisfiable.
\end{proposition}

The above propositions establish soundness of the NRF-SAT algorithm. In other words, if the algorithm terminates and returns an answer then it is a valid MAPF solution. However the termination needs a separate investigation.

\begin{proposition}
The abstraction refinement loop for a specified $\mathit{SoC}$ is executed by the NRF-SAT algorithm finitely many times.
\label{prop-term}
\end{proposition}

\begin{proof}
Each iteration of the abstraction refinement loop corresponds to a counterexample, a conflict between a pair of agents. This conflict appears between two paths extracted from a pair of DAGs say $\mathit{DAG}_i$ and  $\mathit{DAG}_j$. These two specific DAGs cannot be interpreted from the satisfying assignment of $\mathcal{F''}(\mathit{SoC})$ again in any of the next iterations of the abstraction refinement loop since the conflict elimination constraint forbids them to appear simultaneously, either $\mathit{DAG}_i$ or $\mathit{DAG}_j$ must be different. Since there is finitely many pairs of DAGs $\mathit{DAG}_i$ and  $\mathit{DAG}_j$ the algorithm either forbids them all during the refinements or terminates earlier.
\end{proof}

We are now ready to state the main theoretical property of the NRF-SAT algorithm.

\begin{theorem}
The NRF-SAT algorithm returns sum-of-costs optimal solution for a solvable input MAPF instance.
\end{theorem}

\begin{proof}
For a solvable MAPF instance, the NRF-SAT finds the first sum-of-costs $\mathit{SoC}$ for which the abstraction refinement loop finishes with a solution since all previous loops are guaranteed to terminate due to Proposition \ref{prop-term}. Due to monotonicity of solvability of bounded MAPF with respect to the sum-of-costs, this solution is optimal. 
\end{proof}

\begin{figure}[h]
    \centering
    \includegraphics[trim={3.0cm 23.0cm 4.5cm 3.0cm},clip,width=0.5\textwidth]{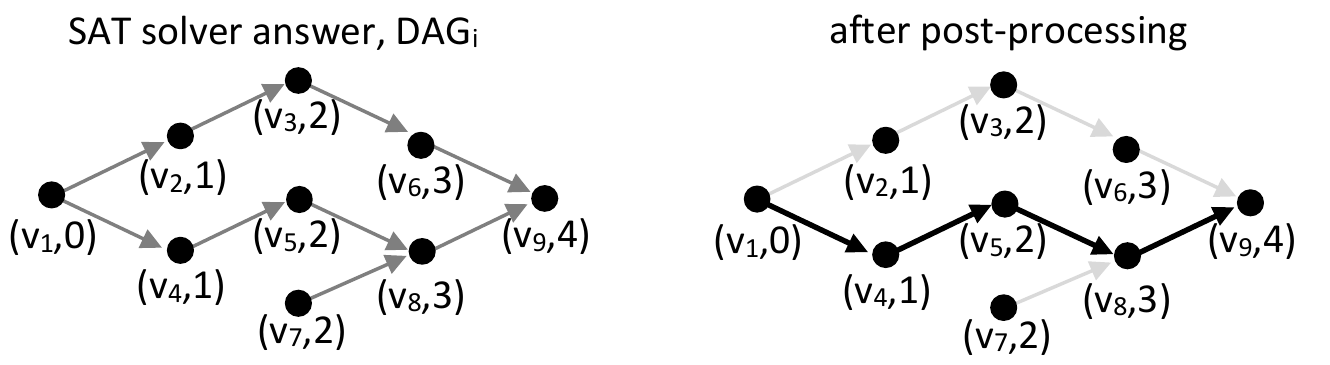}
    \caption{Post-processing step in which path is extracted from DAG answered by the SAT solver.}
    \label{figure-path-extraction}
\end{figure}

\begin{figure*}[t]
    \centering
    \begin{subfigure}{0.33\textwidth}
       \includegraphics[trim={1.5cm 6.5cm 1.0cm 8.0cm},clip,width=1.0\textwidth]{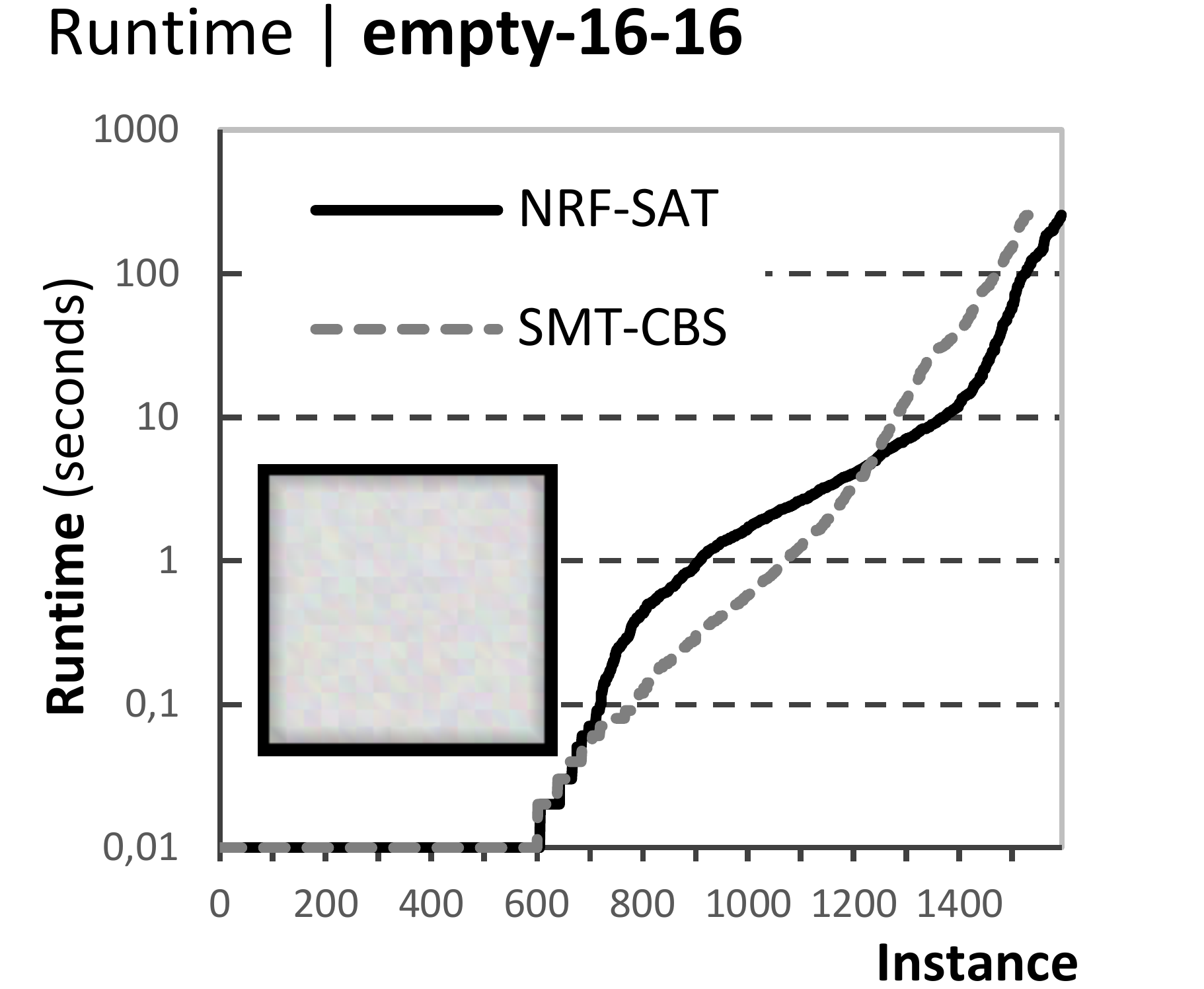}
    \end{subfigure}
    \begin{subfigure}{0.33\textwidth}
       \includegraphics[trim={1.5cm 6.5cm 1.0cm 8.0cm},clip,width=1.0\textwidth]{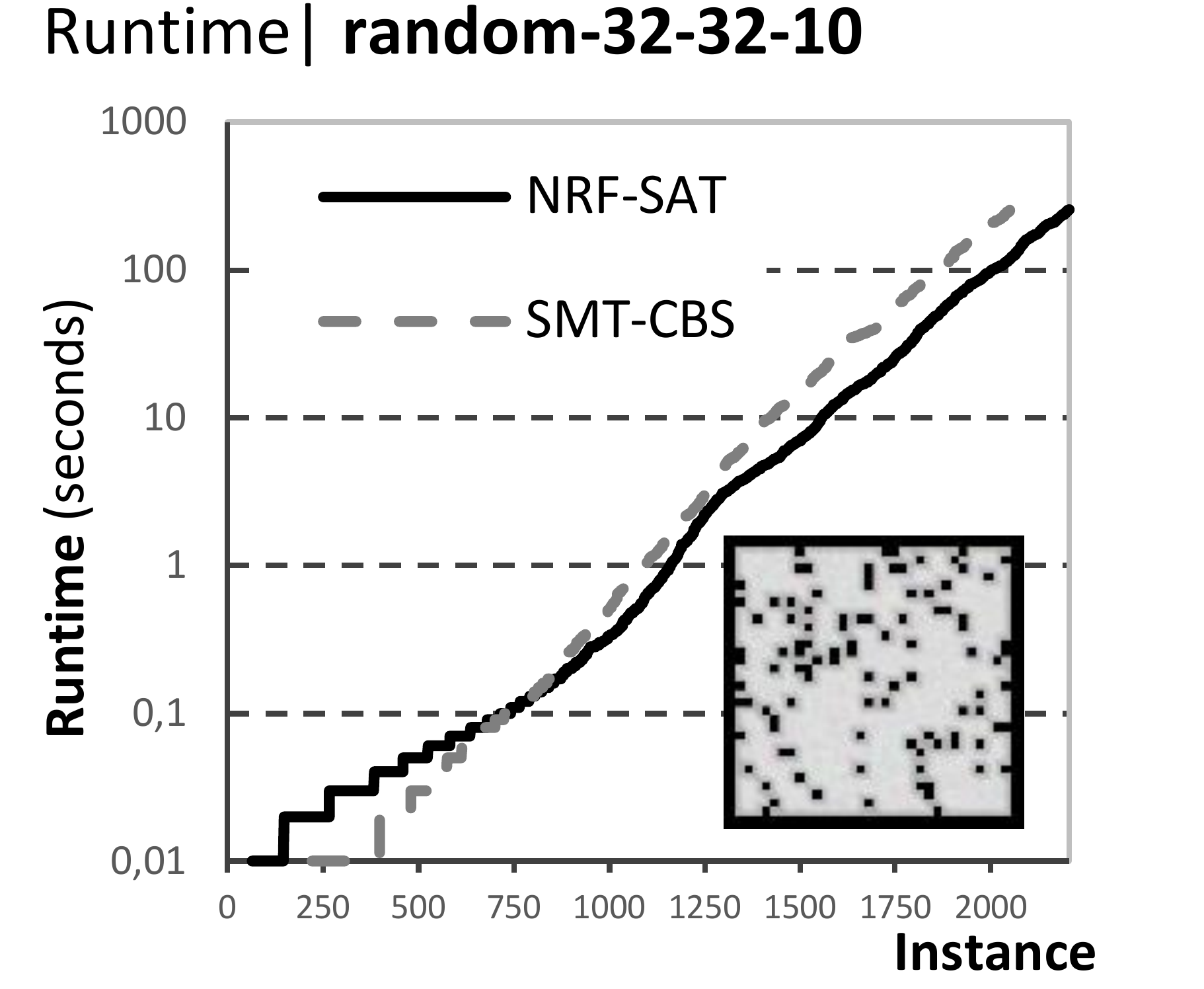}
    \end{subfigure}
    \begin{subfigure}{0.33\textwidth}
       \includegraphics[trim={1.5cm 6.5cm 1.0cm 8.0cm},clip,width=1.0\textwidth]{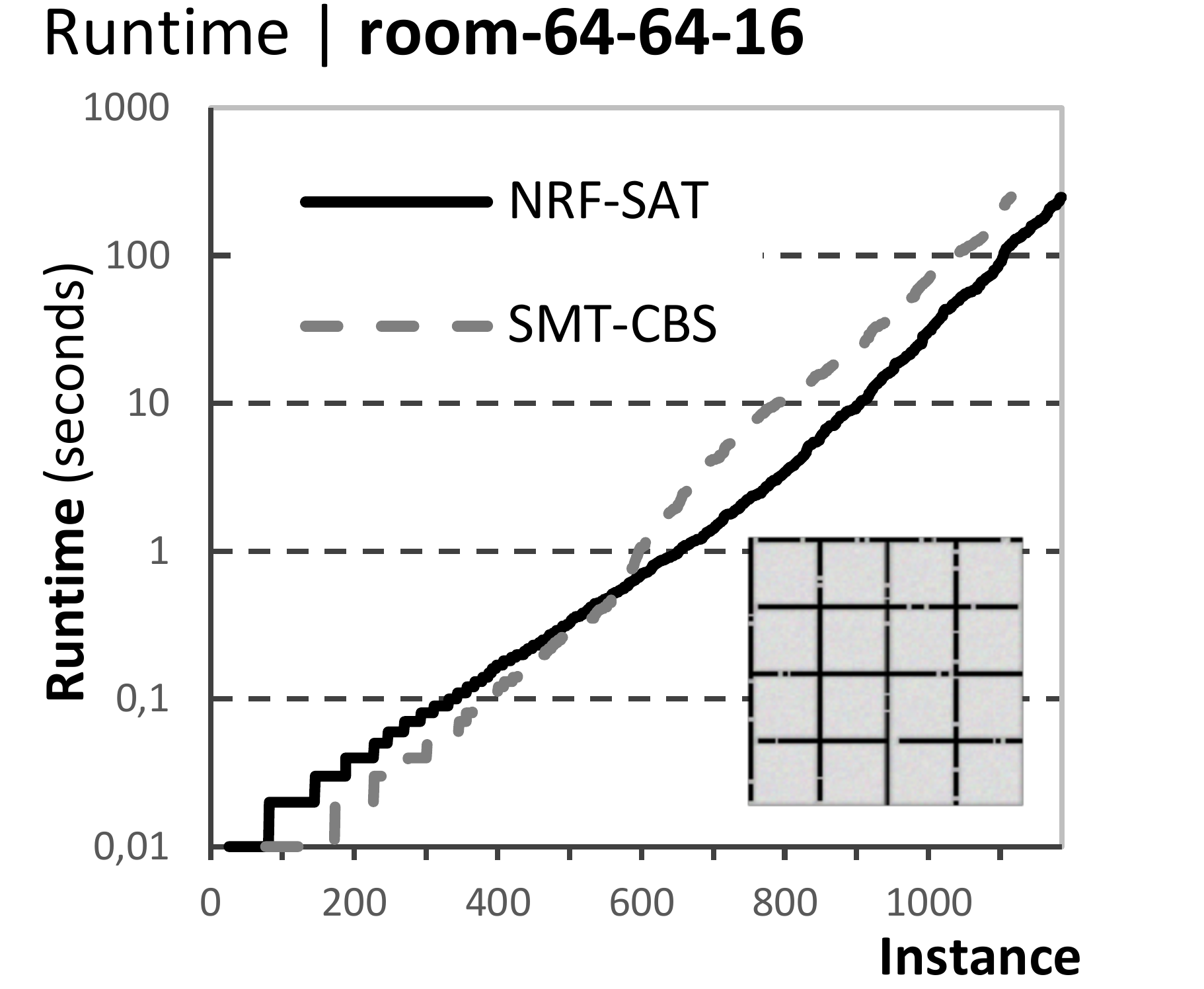}
    \end{subfigure}        
    \caption{Runtime comparison of two SAT-based MAPF solvers NRF-SAT and SMT-CBS. Cactus plots of runtimes for the solvers are shown (lower plot means better performance).}
    \label{expr-sat}
\end{figure*}

\section{Experimental Evaluation}

We performed an experimental evaluation of NRF-SAT on a number of MAPF benchmarks from \texttt{movingai.com} \cite{DBLP:journals/tciaig/Sturtevant12}. We compared NRF-SAT against SAT-based SMT-CBS \cite{DBLP:conf/ijcai/Surynek19} and CSP-based LazyCBS \cite{DBLP:conf/aips/GangeHS19} which are the solvers using similar compilation-based approach to MAPF.

\subsection{Benchmarks and Setup}

We implemented NRF-SAT in C++ via reusing the code of the original implementation of SMT-CBS. Both SAT-based MAPF solvers are built on top the Glucose 3 SAT solver \cite{DBLP:conf/ijcai/AudemardS09,DBLP:journals/ijait/AudemardS18} still ranking among top performing SAT solvers according to relevant competitions \cite{DBLP:journals/ai/FroleyksHIJS21}. The extraction of agent's path from DAG is implemented as a simple breadth-first search.

The important feature of the Glucose SAT solver is that it provides an interface for adding clauses incrementally that is employed during abstraction refinements. After refining the formula being answered by the SAT solver with new clauses, the solving process does not need to start from scratch. Instead the learned state of the solver is utilized in its run after the formula refinement which usually speeds up the process.

As of LazyCBS, we used its original implementation in C++. LazyCBS is built on top of the Geas CSP solver that supports lazy clause generation \cite{DBLP:conf/cpaior/Stuckey10}, a feature used by LazyCBS to eliminate MAPF conflicts lazily.

To obtain instances of various difficulties we varied the number of agents from 1 to $K$, where $K$ is the maximum number of agents for which at least one solver is able to solve some instance in the given time limit of 300 seconds (5 minutes). $K$ varied from approximately 80 agents to 120 agents depending on the benchmark map. For each number of agents, we generated 25 instances according to random scenarios provided on \texttt{movingai.com} (for each benchmark map we generated $25 \times K$ instances, i.e. approximately 2500 MAPF instances per map).

All experiments were run on a system consisting of Xeon 2.8 GHz cores, 32 GB RAM per solver instance, running Ubuntu Linux 18 \footnote{To provide reproducibility of presented results the complete source code of NRF-SAT is available on \texttt{https://github.com/surynek/boOX}.}. 

\begin{figure}[h]
    \centering
    \includegraphics[trim={2.4cm 20.0cm 4.5cm 2.5cm},clip,width=0.48\textwidth]{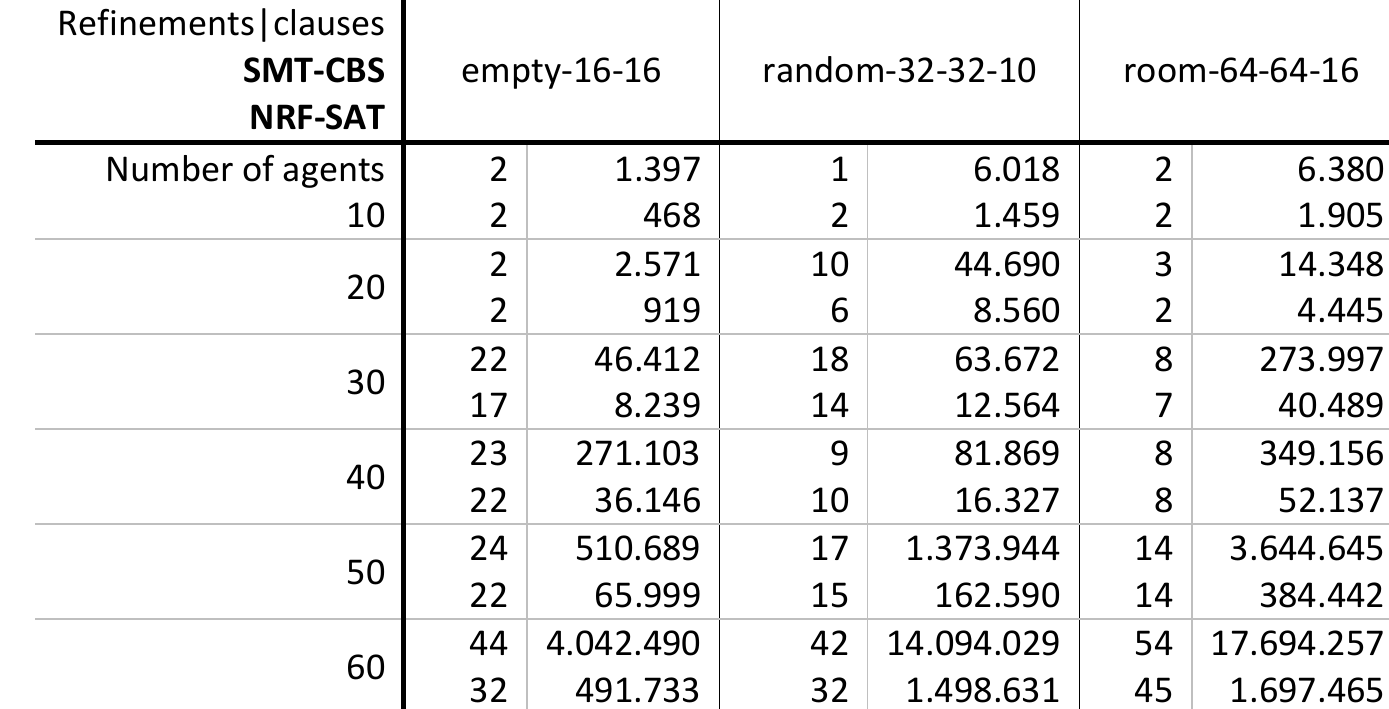}
    \caption{The total number of clauses and refinements of SAT-based solvers SMT-CBS and NRF-SAT.}
    \label{tab-clauses-ref}
\end{figure}


\begin{figure*}[t]
    \centering
    \begin{subfigure}{0.33\textwidth}
       \includegraphics[trim={1.5cm 6.5cm 1.0cm 8.0cm},clip,width=1.0\textwidth]{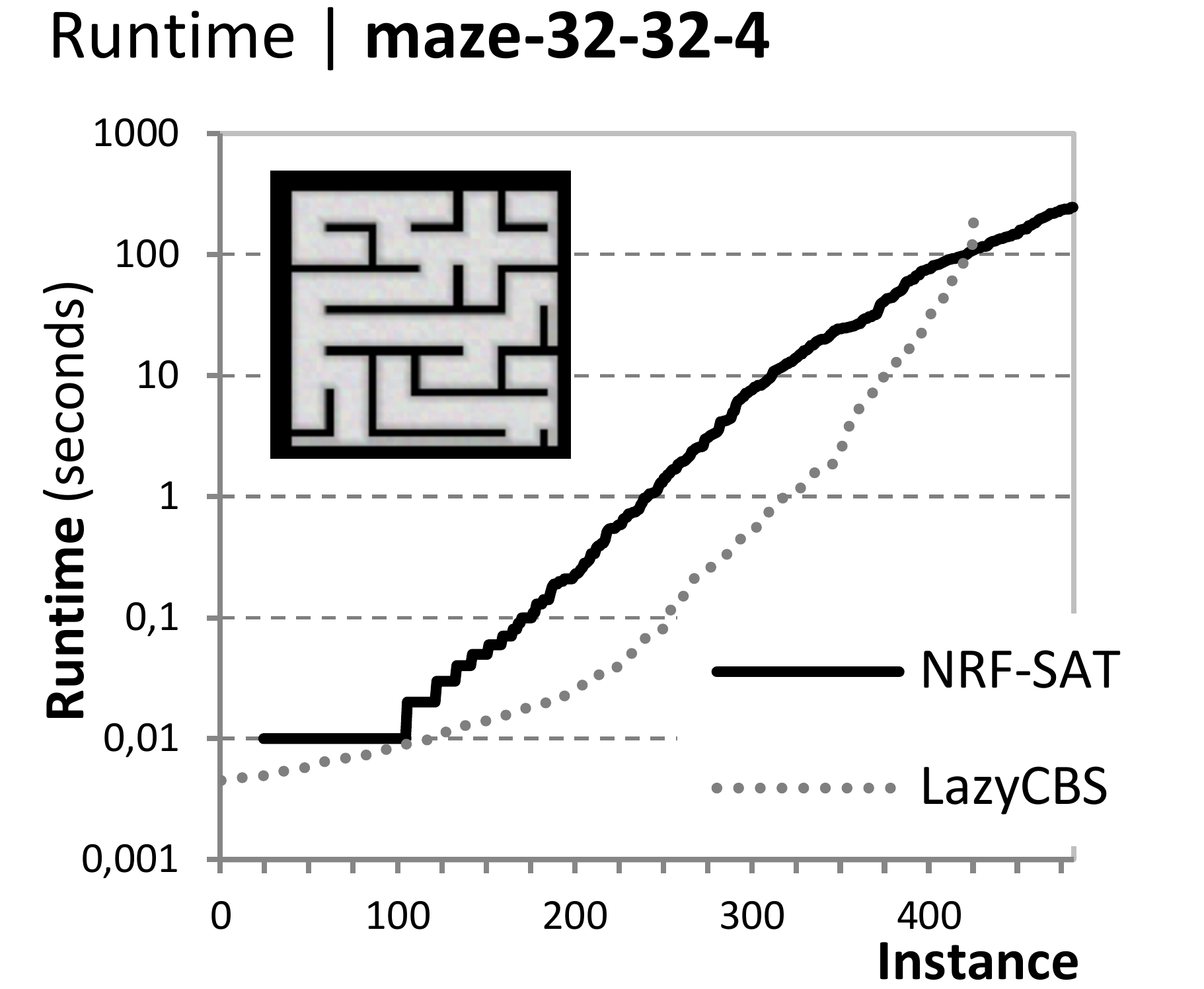}
    \end{subfigure}
    \begin{subfigure}{0.33\textwidth}
       \includegraphics[trim={1.5cm 6.5cm 1.0cm 8.0cm},clip,width=1.0\textwidth]{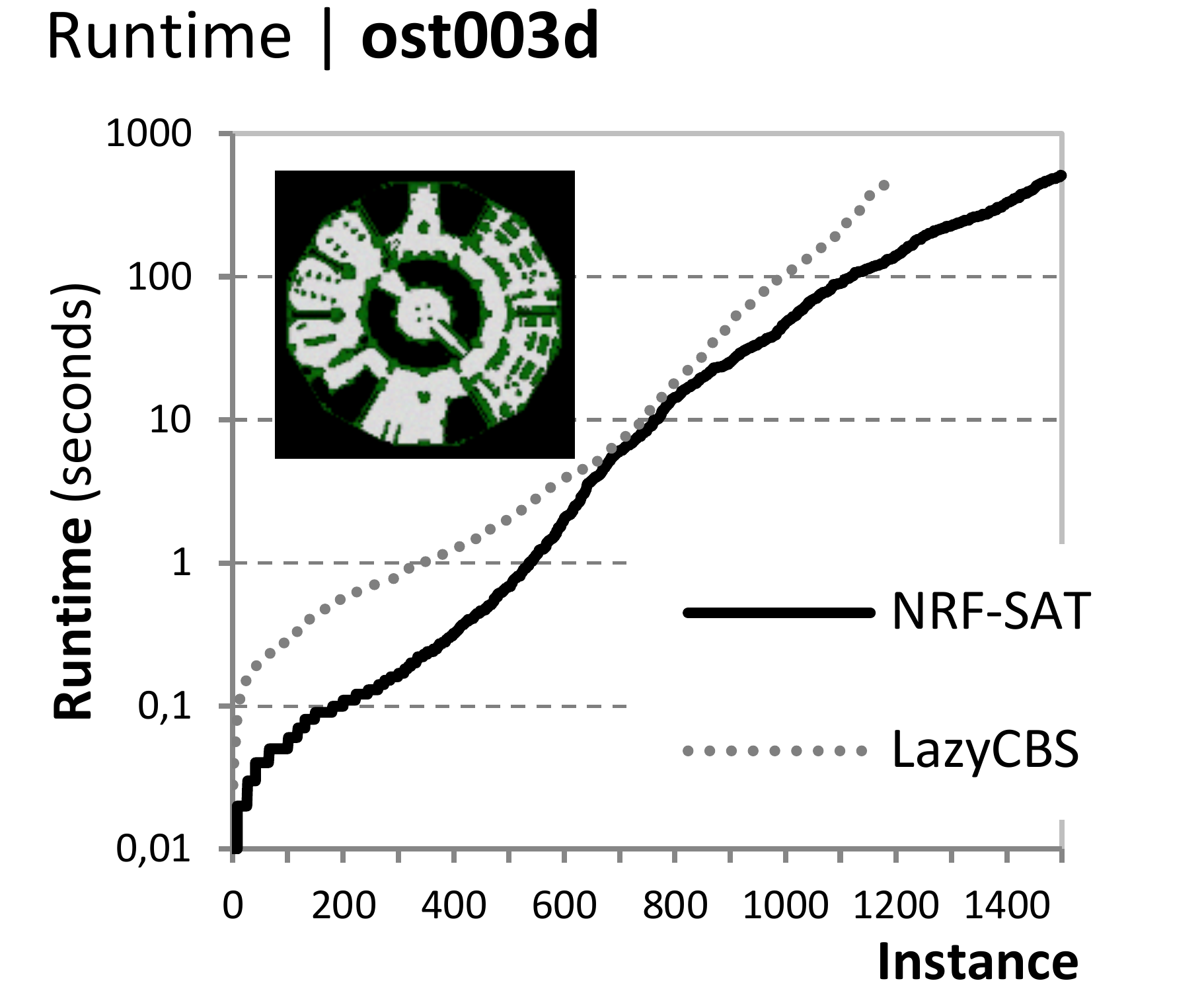}
    \end{subfigure}
    \begin{subfigure}{0.33\textwidth}
       \includegraphics[trim={1.5cm 6.5cm 1.0cm 8.0cm},clip,width=1.0\textwidth]{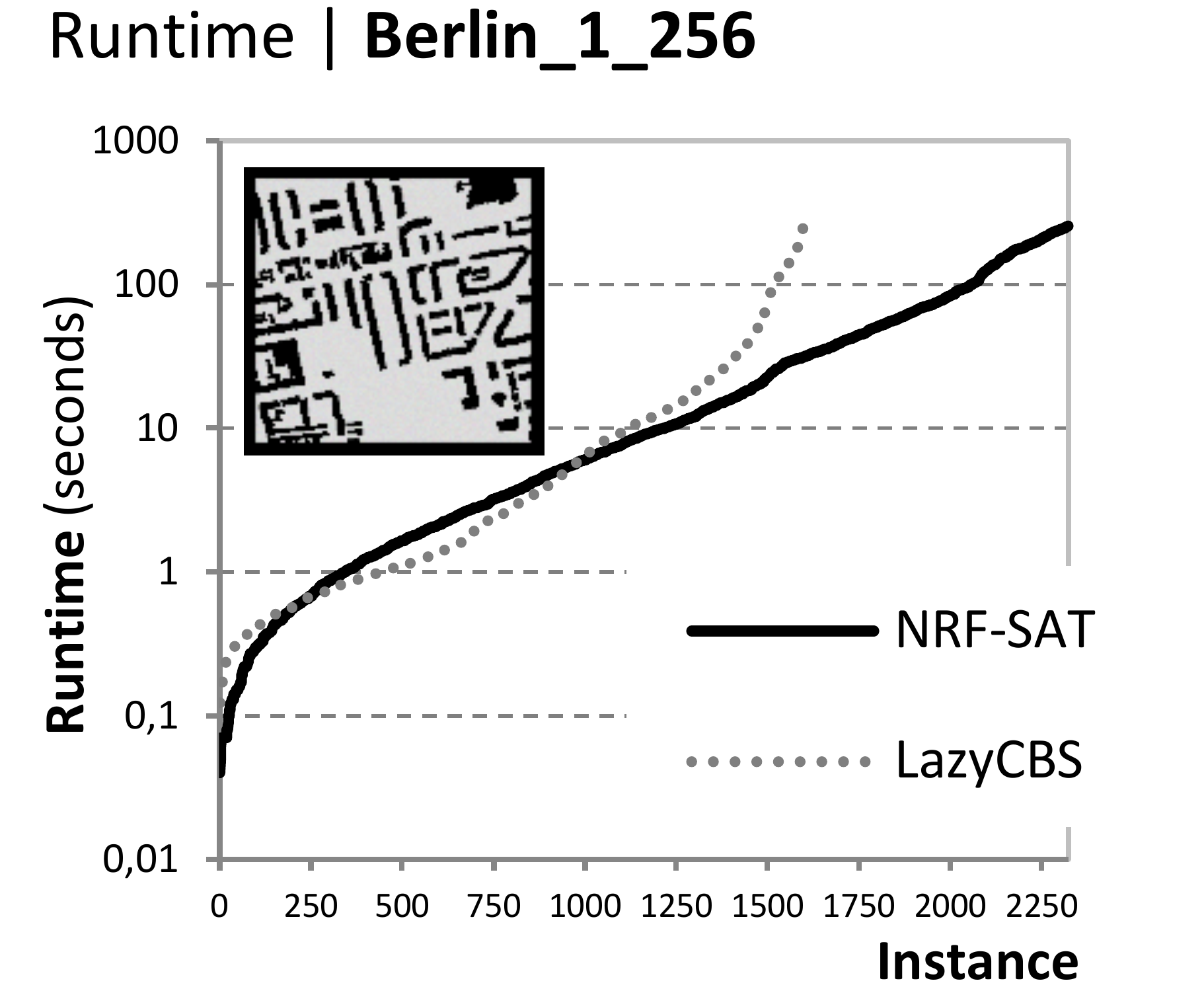}
    \end{subfigure}        
    \caption{Runtime comparison between NRF-SAT and LazyCBS. Cactus plots of runtimes for the solvers are shown (lower plot is better).}
    \label{expr-competitive}
\end{figure*}

\subsection{The Effect of Non-Refining}

We investigated the effect of non-refining w.r.t. the {\em path consistency} constraints in SAT-based solvers. The comparison of the SMT-CBS and NRF-SAT solvers in terms of the number of clauses being generated along the entire solving process is shown in Figure \ref{tab-clauses-ref} (this comparison is not relevant for LazyCBS, hence it is not included in the test).

The table shows the median number of clauses being generated by SMT-CBS for specific map and selected number of agents and the number of clauses generated by NRF-SAT for the same instance. In this test, small to medium sized maps have been used: \texttt{empty-16-16}, \texttt{random-32-32-10}, and \texttt{room-64-64-16}.

We can observe that NRF-SAT generates significantly fewer clauses than SMT-CBS. This trend is even more pronounced as the number of agents increases. Order of magnitude fewer clauses are generated by NRF-SAT for 60 and more agents on the presented benchmarks.

The explanation of this result that {\em path consistency} constraints encompass many at-most-one constraints that yields many clauses in most of its SAT representations \cite{DBLP:conf/soict/NguyenM15}.

We also report the total number of abstraction refinements for the same set of instances as reported for the number of clauses in Figure \ref{tab-clauses-ref}. Surprisingly non-refining often leads to a significant reduction of the number of abstraction refinements. Although this not a rule as sometimes increase in the number of abstractions in contrast to SMT-CBS can be observed in NRF-SAT, the reduction prevails.

One reason of this difference is that the choice of final paths is done by the SAT solver in SMT-CBS while in NRF-SAT the final choice is made by the path-processing procedure that extracts paths from DAGs in a fixed order which seems to be more suitable for abstraction refinements.

In addition to this, we tested the impact of leaving the abstraction non-refined on the overall performance of the SAT-based MAPF solver. Runtime results comparing SMT-CBS and NRF-SAT on the same set of maps: \texttt{empty-16-16}, \texttt{random-32-32-10}, and \texttt{room-64-64-16} is shown in Figure \ref{expr-sat}.

The runtime results are presented using cactus plots often used to present the results of SAT competitions \cite{DBLP:conf/aaai/BalyoHJ17}, that is runtimes for all instances are sorted so the $x$-th result along the horizontal axis represents the runtime for the $x$-th fastest solved instance by the given MAPF solver. The lower plot for the solver means better performance.

Instances with up to approximately 70 agents for \texttt{empty-16-16}, 110 agents for \texttt{random-32-32-10}, and 60 agents for \texttt{room-64-64-16} were solved by the solvers in the given time limit. SMT-CBS solved 1564, 2120, and 1152 and NRF-SAT solved 1633, 2266, and 1203 in total for \texttt{empty-16-16}, \texttt{random-32-32-10}, and \texttt{room-64-64-16} respectively, that is, significantly more instances for NRF-SAT. It need to be taken into account that the instances that were additionally solved by the NRF-SAT solver rank among the difficult ones.

We can observe in Figure \ref{expr-sat} that NRF-SAT dominates in harder instances while in easier instances SMT-CBS is sometimes better (most prominently on \texttt{empty-16-16}).

The explanation for the better performance of NRF-SAT is twofold. First, it generates significantly smaller formulae across entire abstraction refinement process hence the processing time itself is shorter. Second, the resulting formula with omitted path consistency constraints is easier to solve by the SAT solver which coupled with the fact the total number of abstraction refinements tends to be smaller in NRF-SAT leads to overall better performance.

\subsection{Competitive Comparison}

The competitive comparison of NRF-SAT and LazyCBS in terms of runtime is shown in Figure \ref{expr-competitive}. This comparison is focused on benchmarks that were identified by previous studies as those where compilation-based MAPF solvers perform well \cite{DBLP:conf/socs/KaduriBS21}. These benchmarks include those on mazes, city maps, and game maps. The representatives we selected for presentation are: \texttt{maze-32-32-4} (a maze map),  \texttt{ost003d} (a game map), \texttt{Berlin\_1\_256} (a city map). Additional experiments we made on other maps from these categories yield similar results.

Again, the runtime results are presented using cactus plots. The summary of the results is that LazyCBS solved 527, 1193, and 1599 while NRF-SAT solved 582, 1335, and 2321 in total for \texttt{maze-32-32-4}, \texttt{ost003d}, and \texttt{Berlin\_1\_256} respectively, that is NRF-SAT solver significantly more instances where again these extra instances rank among difficult ones.

Close look at the results reveals that the general trend is that LazyCBS has slightly sharper increase in runtimes as instances are getting harder. This results in reaching the timeout by LazyCBS sooner while NRF-SAT can still solve the instances within the time limit. The advantage of NRF-SAT tends to be more significant as the size of the map grows which is surprising for the SAT-based MAPF solvers that are notorious to struggle on large maps.

The explanation for the better performance of NRF-SAT on the presented benchmarks is that non-refining in the CEGAR architecture is especially helpful when dealing with large maps where it leads to significantly smaller formulae than in previous SAT-based solvers. Moreover this combined with the rest of the CEGAR architecture leads to a competitive solver.

There are many other benchmarks where LazyCBS performs significantly better than NRF-SAT. Hence we do not claim that NRF-SAT is state-of-the-art solver for MAPF. However, as we report, there are several important domains where NRF-SAT achieves competitive performance which shows the importance of non-refined abstractions. Moreover, it is needed to take into account that NRF-SAT is in fact a vanilla solver for MAPF with no specific MAPF techniques such as {\em symmetry breaking} \cite{DBLP:conf/aips/0001GHS0K20} or {\em rectangle reasoning} \cite{DBLP:conf/ijcai/LiFB0K19} being used. Hence there is a potential that the SAT-based MAPF solvers can return among top performing MAPF solvers at least in certain domains and the CEGAR architecture with non-refined abstractions can contribute to this.

\section{Conclusion}
We proposed a novel solver called NRF-SAT for MAPF based on the CEGAR architecture and Boolean satisfiability that uses non-refined abstraction during the solving process.

Unlike previous uses of CEGAR in MAPF we not only eliminate conflicts between agents via abstraction refinements but we also further strengthen the initial abstraction. Particularly our new solver NRF-SAT omits large group of constraints in the initial abstraction and never makes refinements with respect to them which as we show can be mitigated by a fast post-processing step.

From a broader perspective, we not only apply the CEGAR architecture with non-refined abstractions for compilation-based MAPF solving but we also generalize the architecture itself. This generalization consists in finding a solution of a different, more general task, than is the original one using the target formalism, rather than finding a solution of the original task using the target formalism directly.

	
\section*{Acknowledgments}
This research has been supported by GA\v{C}R - the Czech Science Foundation, grant registration number 22-31346S.

\bibliographystyle{named}
\bibliography{references}

\end{document}